\documentclass[12pt]{article}

\usepackage{amssymb}
\usepackage{cmupint}
\usepackage{geometry}
\usepackage{parskip} % load before hyperref else vertical space above theorems &c dies
\usepackage{hyperref}
\usepackage{mathtools}

\DeclareSymbolFont{greeksymbols}{U}{jkpssmia}{m}{it}
\DeclareMathSymbol{\lambda}{\mathalpha}{greeksymbols}{21}
\DeclareMathSymbol{\sigma}{\mathalpha}{greeksymbols}{27}

\DeclareFontFamily{U}{matha}{\hyphenchar\font45}
\DeclareFontShape{U}{matha}{m}{n}{
      <5> <6> <7> <8> <9> <10> gen * matha
      <10.95> matha10 <12> <14.4> <17.28> <20.74> <24.88> matha12
      }{}
\DeclareSymbolFont{matha}{U}{matha}{m}{n}
\DeclareMathSymbol{\cap}           {2}{matha}{"58}
\DeclareMathSymbol{\cup}           {2}{matha}{"59}

\DeclareMathSymbol{\sim}           {3}{matha}{"12}

\DeclareFontFamily{U}{mathx}{\hyphenchar\font45}
\DeclareFontShape{U}{mathx}{m}{n}{
      <5> <6> <7> <8> <9> <10>
      <10.95> <12> <14.4> <17.28> <20.74> <24.88>
      mathx10
      }{}
\DeclareSymbolFont{mathx}{U}{mathx}{m}{n}
\DeclareMathDelimiter{\lbrace}     {4}{matha}{"74}{mathx}{"20}
\DeclareMathDelimiter{\rbrace}     {5}{matha}{"75}{mathx}{"28}
\usepackage[cal=boondoxupr,scr=kp]{mathalpha}

\everymath{\displaystyle}

\makeatletter
\AtBeginDocument{
  \check@mathfonts
  \setlength{\fontdimen17\textfont2}{\fontdimen16\textfont2}
}
\makeatother

\let\leq\leqslant
\let\geq\geqslant

\let\epsilon\varepsilon
\let\theta\vartheta
\let\phi\varphi

\newcommand{\ovl}{\overline}

\newcommand{\E}{\mathbb{E}}
\renewcommand{\P}{\mathbb{P}}

\newcommand{\R}{\mathbb{R}}

\newcommand{\F}{\mathscr{F}}

\renewcommand{\b}{\mathcal{b}}
\newcommand{\n}{\mathcal{n}}
\newcommand{\w}{\mathcal{w}}

\DeclareMathOperator{\diam}{diam}
\DeclareMathOperator{\sg}{sg}

\newcommand{\fcirc}{\tilde{f}}
\newcommand{\kcirc}{\tilde{K}}
\newcommand{\lcirc}{\ell}
\newcommand{\prefac}{(2\pi)^{-m}}

\usepackage{amsthm}

\theoremstyle{plain}

\newtheorem{theorem}{Theorem}
\newtheorem{lem}{Lemma}
\newtheorem{prop}{Proposition}

\newtheorem*{main}{Main Theorem}

\theoremstyle{definition}

\newtheorem{defn}{Definition}

\theoremstyle{remark}

\newtheorem*{rem}{Remark}

\makeatletter
\@addtoreset{lem}{section}
\makeatother

\let\mathds\mathbb

\title{Efficient uniform approximation using \\ Random Vector Functional Link networks}
\author{Palina Salanevich\\ Utrecht University\\
\texttt{p.salanevich@uu.nl}
\and
Olov Schavemaker\\ Utrecht University\\
\texttt{o.p.schavemaker@uu.nl}}

\begin{document}
\maketitle
\raggedright

\begin{abstract}
A Random Vector Functional Link (RVFL) network is a depth-2 neural network with random inner weights and biases. Only the outer weights of such an architecture are to be learned, so the learning process boils down to a linear optimization task, allowing one to sidestep the pitfalls of nonconvex optimization problems. In this paper, we prove that an RVFL with ReLU activation functions can approximate Lipschitz continuous functions in $L_\infty$ norm. To the best of our knowledge, our result is the first approximation result in $L_\infty$ norm using nice inner weights; namely, Gaussians. We give a nonasymptotic lower bound for the number of hidden-layer nodes to achieve a given accuracy with high probability, depending on, among other things, the Lipschitz constant of the target function, the desired accuracy, and the input dimension. Our method of proof is rooted in probability theory and harmonic analysis.
\end{abstract}

\section{Introduction}

In this paper, we examine the approximation capacity of the Random Vector Functional Link (RVFL) network. An RVFL is a depth-2 neural network with random inner weights and biases. More precisely, an RVFL is a random function ${N_n:\hfill \R^m \to \R}$ of the form
\begin{align*}
N_n(x)=\sum_{j=1}^n a_j\rho(\langle\w_{\mkern-1mu j},x\rangle+\b_{\mkern-1mu j}),
\end{align*}
where $\w_{\mkern-1mu j}$'s and $\b_{\mkern-1mu j}$'s are iid random variables, $\rho:\R\to\R$ is an \emph{activation function} and $a_j$'s are real numbers, chosen or learned so as to have $N_n$ be close to a target function $f$.

%RVFL models have been successfully applied in practice to, e.g., classification and regression problems, forecasting, non- linear system identification, and function approximation; for an overview, see \cite{malik}. In spite of their success in applications,  theoretical foundation is still lacking \cite[{\S}1]{rvfl}. This paper aims to remedy this discrepancy.

In spite of the simplicity of the architecture, RVFL models have been applied in signal classification and regression problems \cite{katuwal20181ensemble}, forecasting \cite{tang2018noniterative}, time-series data prediction \cite{chen1999rapid} among others; for an overview, see \cite{malik}. Yet, theoretical foundations for RVFL networks are still lacking \cite[{\S}1]{rvfl}. This paper aims to remedy this discrepancy.

Since only the outer weights of RVFL architectures need to be optimized, in practice the learning process boils down to a linear optimization task. Indeed, given training data $\{x_p\}_{p=1}^k,$ we aim to choose  $a_j$'s such that

\begin{gather*}
\begin{Bmatrix}
f(x_1)\\
\vdots\\
f(x_k)
\end{Bmatrix}
\approx
\begin{Bmatrix}
\textstyle\sum_{j=1}^n a_j\rho(\langle\w_{\mkern-1mu j},x_1\rangle+\b_{\mkern-1mu j})\\
\vdots\\
\textstyle\sum_{j=1}^n a_j\rho(\langle\w_{\mkern-1mu j},x_k\rangle+\b_{\mkern-1mu j})
\end{Bmatrix}
=\\
\begin{Bmatrix}
\rho(\langle\w_1,x_1\rangle+\b_1)&\cdots&\rho(\langle\w_n,x_1\rangle+\b_n)\\
\vdots&\ddots&\vdots\\
\rho(\langle\w_1,x_k\rangle+\b_1)&\cdots&\rho(\langle\w_n,x_k\rangle+\b_n)
\end{Bmatrix}
\begin{Bmatrix}
a_1\\
\vdots\\
a_n
\end{Bmatrix}.
\end{gather*}
Since the learning process boils down to a linear optimization problem, training RVFL networks sidesteps the usual pitfalls of nonconvex optimization problems, such as slow convergence and getting stuck in local optima \cite{malik}.

In order to compare our main result to the existing literature, it will be useful to state a version of our main theorem where expressions depending only on $m$ are hidden; we will use $\lesssim_m$ and $\gtrsim_m$ to indicate this.
\begin{main}
Let $K\subset\R^m$ be a convex body with circumradius $R.$ Let $p$ be a circumcenter of $K$ and $f:K\to\R$ be $\ell$-Lipschitz with $\max f-\min f=2M$ and $\zeta=M+\min f.$ Then there exist outer weights so that the corresponding RVFL network $N_n$ with $n$ hidden-layer nodes, ReLU activation function, inner weight distribution $N(0,\sigma I_m)$ with $\sigma>0,$ and biases that are uniformly distributed on $\Bigl[-\sigma R\sqrt{m},\sigma R\sqrt{m}\Bigr]$ satisfies $\P\biggl\{\max_{x\in K\!+p}\lvert f(x-p)-\zeta-N_n(x)\rvert>\epsilon\biggr\}\lesssim_m$
\begin{align*}
\lvert K\rvert\Bigl(2e\ell R\sqrt{e}\Bigr)^{\!m(m+1)}\exp(d(K)m\ln 2)\epsilon^{km}
\end{align*}
for large $n$ and any $\epsilon,k>0,$ provided $(k+m+2)\ln(1/\epsilon)\geq 2$ and $n\gtrsim_m$
\begin{gather*}
(k+m+2)\exp(d(K)4\ln 2)(\ell R/\epsilon)^{2(m+3+1/m)}\ln(1/\epsilon)(1+\theta)^2
\end{gather*}
where $d(K)\in[1,m]$ is defined in Definition \ref{def:dimK} below and $\theta=o(\epsilon)$ as $\epsilon\downarrow 0.$
\end{main}
While many similar results exist in the literature, ours improves on all of them in at least one way. To the best of our knowledge, the two papers most like ours are \cite{rvfl} and \cite{hsu}; both of them also examined how many hidden-layer nodes suffice for an RVFL to be able to approximate (Lipschitz) continuous functions.

In \cite{rvfl}, Needell et al.\ also obtained a sufficient number of hidden-layer nodes depending \emph{superexponentially} on the dimension. They consider $\ell_2$ approximation, however, and the supports of their uniformly distributed weights and biases do crucially depend on the approximation error, making their distributions not very suitable for practical use.

Our Gaussian weights, which are more common in practice than uniform ones, do not have variances depending on the approximation error. Our result additionally improves on \cite[Thm 4.1]{rvfl} in that we measure the error in the $L_\infty$ norm. In fact, our paper is, as far as we know, the first one to combine $L_\infty$ approximation and Gaussian weights.

Hsu et al.\ not only managed to obtain a sufficient number of hidden-layer nodes that depends exponentially on the dimension, but also a necessary number of hidden-layer nodes, albeit in $\ell_2$ norm  \cite{hsu}. However, unlike our Gaussian weights, which align with common practice, their weight distribution is supported on a discrete subset of the unit sphere, which may be too restrictive for practical use. Their proof method, whilst conceptually similar to ours, differs in many details, such as using Fourier series where we use the Fourier transform.

We would also like to draw attention to \cite{nagaraj}, which also concerns itself with RVFL networks, albeit without the RVFL moniker. Whereas the ``corrective method'' developed therein is very different from our proof method, their use of spectral methods has been a great inspiration for this paper.

Lastly, many technical proof details for the lemmas below have been delegated to the appendix for ease of readability. We highly encourage the interested reader to read the appendix as well.

\subsection*{Conventions}

\begin{itemize}
\item Square brackets may denote Iverson brackets.
\item Derivatives may be denoted by a dot atop the function.
\item $j_\nu$ is the first positive zero of the Bessel function $J_\nu.$
\item $a\approx-2.3381$ symbolizes the first negative zero of the Airy function Ai.
\item $\rho$ denotes the ReLU function.
\item $\sg$ denotes the sign function.
\item Integrals without specified integration domains are understood to integrate over all of Euclidean space.
\item The Fourier transform $\F\{\phi\}(v)=\textstyle\int\displaystyle\phi(u)\exp(-i\langle v,u\rangle)\,du.$
\item $\delta_X$ denotes the pdf of a random variable $X$ (cf., $\delta_0).$
\item We write $\phi(\diamondsuit)$ in lieu of the more common $\phi(\cdot).$
\item $V_m$ denotes the volume of the $m$-dimensional unit ball.
\item Absolute value bars may denote either the $m$-dimensional Lebesgue measure or the $\ell_2$ norm in any dimension.
\item $\lVert\diamondsuit\rVert_K=\max_K\,\lvert\diamondsuit\rvert.$
\item Always $\ell>0.$
\end{itemize}

\section{Bird's-eye overview}
\label{sec:birdeye}

Our approximation procedure essentially comprises four steps. We first extend the $\ell$-Lipschitz target function $f$, which is only defined on some convex compactum $K,$ to a compactly supported $\ell$-Lipschitz function $\fcirc.$ The next three steps comprise a chain of approximations: $\fcirc\approx g\approx h\approx N_n.$

The first of those steps is to approximate $\fcirc$ with a ``smoothed'' version $g$ obtained by convolving $\fcirc$ with a specific nascent delta function. We will see that $g$ can be viewed as an ``infinite width'' depth-2 neural network with Gaussian inner weights.  The biases of $g,$ however, are not random, and the activation function is a cosine function.

Thereafter, we approximate $g$ by $h;$ that is, an ``infinite width'' RVFL with ReLU activation functions and Gaussian inner weights, and may be seen as the ``infinite width limit'' of the desired RVFL.

Lastly, we will use a concentration inequality to approximate $h$ by a finite width counterpart $N_n$ with $n$ hidden-layer nodes.

\section{Main result and proof}

Before we state our main theorem, the following definitions will prove useful.
\begin{defn}
A \emph{convex body} is a convex compactum with nonempty interior.
\end{defn}
\begin{defn}
Let $K\subset\R^m$ be a nonempty compactum. Its \emph{circumradius} is
\begin{align*}
R=\inf_p\max_{u\in K}\,\lvert u-p\rvert.
\end{align*}
The unique $p\in\R^m$ achieving the infimum is called the \emph{circumcenter}.
\end{defn}
\begin{theorem}
Let $K$ be a nonempty compactum with circumradius $R.$ Then there exists a unique $p$ satisfying $R=\max_{u\in K}\,\lvert u-p\rvert.$ If $K$ is a convex body, $p\in K.$
\end{theorem}
\begin{proof}
If $v\notin K+\{x\in\R^m:\lvert x\rvert\leq\diam(K)\},$ then
\begin{align*}
\max_{u\in K}\,\lvert u-v\rvert>\diam(K)\geq\inf_p\max_{u\in K}\,\lvert u-p\rvert,
\end{align*}
so the infimum is achieved for some $p\in K+\{x\in\R^m:\lvert x\rvert\leq\diam(K)\}.$ As for uniqueness, suppose both $p$ and $v$ achieve the infimum. Then
\begin{align*}
K&\subset\{x\in\R^m:\lvert x-p\rvert\leq R\}\cap\{x\in\R^m:\lvert x-v\rvert\leq R\}\\
&\subset\{x\in\R^m:\lvert x-(v+p)/2\rvert\leq\sqrt{\displaystyle R^2-\lvert v-p\rvert^2/4}\},
\end{align*}
but $R$ is the circumradius, so $\sqrt{\displaystyle R^2-\lvert v-p\rvert^2/4}\geq R\Rightarrow\lvert v-p\rvert=0.$

If $K$ is a convex body and $\diam(K)=\lvert u-v\rvert$ with $u,v\in K,$ the circumcenter $p=(u+v)/2\in K$ by convexity.
\end{proof}
Note that $R>0$ if $K$ has at least two elements.
\begin{defn}
\label{def:dimK}
Let $K\subset\R^m$ be a compactum with at least two elements and circumradius $R$. We define
\begin{align*}
d(K)=\lg\frac{\displaystyle\lvert K+\{x\in\R^m:\lvert x\rvert\leq R\}\rvert}{\displaystyle\lvert\{x\in\R^m:\lvert x\rvert\leq R\}\rvert},
\end{align*}
where $\lg$ is the binary logarithm.
\end{defn}
\begin{theorem}
If $K\subset\R^m$ is a compactum with at least two elements, then $d(K)\in[1,m].$
\end{theorem}
\begin{proof}
If $K$ has circumcenter $p$ and $2R=\diam(K)=\lvert u-v\rvert$ with $u,v\in K,$ then $\{u,v\}\subset K\subset\{x\in\R^m:\lvert x-p\rvert\leq R\},$ so
\begin{gather*}
1=d(\{u,v\})\leq d(K)\leq d(\{x\in\R^m:\lvert x-p\rvert\leq R\})=m
\end{gather*}
because $\{x\in\R^m:\lvert x-u\rvert\leq R\}\cap\{x\in\R^m:\lvert x-v\rvert\leq R\}$ is just the singleton containing $p.$
\end{proof}
Thus, $d(K)$ may be interpreted as some sort of unfamiliar \emph{dimensional quantity.}

Our main result is the following.
\begin{main}
Let $K\subset\R^m$ be a convex body with circumradius $R.$ Let $p$ be a circumcenter of $K$ and $f:K\to\R$ be $\ell$-Lipschitz with $\max f-\min f=2M$ and $\zeta=M+\min f.$ Then there exist outer weights so that the corresponding RVFL network $N_n$ with $n$ hidden-layer nodes, ReLU activation function, inner weight distribution $N(0,\sigma I_m)$ with $\sigma>0,$ and biases that are uniformly distributed on $\Bigl[-\sigma R\sqrt{m},\sigma R\sqrt{m}\Bigr]$ satisfies $\P\Bigl\{\lVert f(\diamondsuit-p)-\zeta-N_n\rVert_{K\!+p}>\epsilon\Bigr\}\lessapprox$
\begin{align*}
\lvert K\rvert\Bigl(2e\ell R\sqrt{e}\Bigr)^{\!m(m+1)}\exp(d(K)m\ln 2)\epsilon^{km}\cdot
2\Theta_m\biggl(\frac{1}{\pi\sqrt{2}}\exp(\tfrac{1}{6}-2^{-2/3}am^{1/3})\biggr)^{\!\!m}
\end{align*}
for large $n$ and any $\epsilon,k>0,$ provided $(k+m+2)\ln(1/\epsilon)\geq 2$ and $n\geq$
\begin{gather}
\begin{gathered}\label{eq:n_bound}
(k+m+2)\exp(d(K)4\ln 2)(\ell R/\epsilon)^{2(m+3+1/m)}\ln(1/\epsilon)(1+\theta)^2\\
\frac{8}{\pi}\biggl(\frac{5}{\pi}(2-2^{1/3}am^{-2/3})^2(m^2+3m+1)\exp(3-2^{-2/3}am^{1/3})\biggr)^{\!\!2}(4e)^m
\end{gathered}
\end{gather}
where $\Theta_m\approx m^{3/2}$ and $\theta=o(\epsilon)$ as $\epsilon\downarrow 0$ (see appendix).
\end{main}
We henceforth suppose that $\zeta=0$ so that $-M\leq f\leq M;$ since $\zeta$ may be inter- preted as the bias of the outer layer of $N_n,$ we may do so WLOG. We moreover suppose that $p=0.$ In practice $p$ can be approximated using data as part of preprocessing.

\subsection{Proof of the Main Theorem}

Our first order of business is extending $f$. This will make it easier to construct a smooth approximation of $f$. The following lemma facilitates this:
\begin{lem}\label{lem:lipschitz_extension}
If $U\subset\R^m$ is bounded and $f:U\to\R$ is $\ell$-Lipschitz continuous, $f$ can be extended to a compactly supported $\lcirc$-Lipschitz continuous function $\fcirc.$
\end{lem}
\begin{proof}
The idea is to extend $\lvert f\rvert,$ which is also $\ell$-Lipschitz by the reverse triangle inequality, to the compactly supported $\lcirc$-Lipschitz function $\fcirc_+$ defined by
\begin{align}
\fcirc_+(x)=\rho\biggl(\max_{u\in \ovl{U}}\Bigl(\lvert f(u)\rvert-\ell\lvert x-u\rvert\Bigr)\biggr)\label{eq:fcirc+}
\end{align}
and then to multiply $\fcirc_+$ by an extension of $\sg f$ to get a compactly supported $\ell$- Lipschitz extension $\fcirc$ of $f.$ See the appendix for the details.
\end{proof}
Henceforth $\kcirc\subset K+\{x\in\R^m:\lvert x\rvert\leq M/\ell\}$ will denote the support of $\fcirc$ as defined in the proof of Lemma \ref{lem:lipschitz_extension} in the appendix applied to the $f$ in the Main Theorem.

Now that we have extended $f$, we can define $g,$ which we do as follows:
\begin{align*}
g(x)=\prefac\int F(v)\exp(i\langle v,x\rangle-\lvert v\rvert^2/2\lambda^2)\Psi(v/\lambda)\,dv,
\end{align*}
where $\lambda=\sigma\Lambda$ with $\Lambda>0$ TBD, $F=\F\{\fcirc\},$ and $\Psi(x)=(\omega*\omega)\Bigl(x/\sqrt{m}\Bigr)$ with
\begin{align*}
\omega(x)\propto\Bigl[\lvert x\rvert\leq\tfrac{1}{2}\Bigr]J_\nu\Bigl(2j_\nu\lvert x\rvert\Bigr)\Big/\lvert x\rvert^\nu,
\end{align*}
where $\nu=m/2-1$ and the proportionality constant rescales $\psi=\F^{-1}\{\Psi\}$ into a pdf; see \cite[{\S}5]{ehm} for details. Note that $\omega(0)$ is well-defined because \cite[(10.7.3)]{dlmf}
\begin{align*}
t^{-\nu}J_\nu(t)\to\frac{\displaystyle 2^{-\nu}}{\Gamma(\nu+1)}\text{ as }t\to 0.
\end{align*}
Upon recognizing $g$ as an inverse Fourier transform, we see that $g$ may be inter- preted as the convolution of $\fcirc$ and a nascent delta function, i.e., $g$ is a smoothed version of $\fcirc$, in light of the convolution theorem and the scaling property of the Fourier transform.
\begin{lem}
\label{lem:step1}
$\lVert\fcirc-g\rVert_\infty\leq\frac{\lcirc}{\lambda}\Bigl(2-2^{1/3}am^{-2/3}\Bigr)\sqrt{m}.$
\end{lem}
\begin{proof}
Let $Z\sim N(0,I_m).$ Essentially, we first show that
\begin{align*}
g(x)=\int\fcirc(x-s/\lambda)(\delta_Z*\psi)(s)\,ds
\end{align*}
using standard integral manipulation. Since $\psi$ is a pdf,
\begin{align*}
\lvert\fcirc(x)-g(x)\rvert=\biggl\lvert\fcirc(x)-\int\fcirc(x-s/\lambda)(\delta_Z*\psi)(s)\,ds\biggr\rvert&=\\
\biggl\lvert\int\fcirc(x)(\delta_Z*\psi)(s)\,ds-\int\fcirc(x-s/\lambda)(\delta_Z*\psi)(s)\,ds\biggr\rvert&\leq\\
\int\lvert\fcirc(x)-\fcirc(x-s/\lambda)\rvert(\delta_Z*\psi)(s)\,ds&\leq\\
\frac{\lcirc}{\lambda}\int\lvert s\rvert(\delta_Z*\psi)(s)\,ds,
\end{align*}
so all that remains is to bound the integral. For the remaining details, see the appendix.
\end{proof}
Before we introduce $h,$ it will prove useful to rewrite $g$ in yet another form. Since
\begin{align*}
\prefac\int F(v)\exp(i\langle v,x\rangle-\lvert v\rvert^2/2\lambda^2)\Psi(v/\lambda)\,dv&=\\
g(x)=\int\fcirc(x-s/\lambda)(\delta_Z*\psi)(s)\,ds&\in\R,
\end{align*}
the inverse Fourier transform integral equals its own real part, i.e.,
\begin{align*}
g(x)=\prefac\int\lvert F(v)\rvert c(v,x)\exp(-\lvert v\rvert^2/2\lambda^2)\Psi(v/\lambda)\,dv&=\\
\prefac\lambda^m\int\lvert F(\lambda w)\rvert c(\lambda w,x)\exp(-\lvert w\rvert^2/2)\Psi(w)\,dw&=\\
(2\pi)^{-m/2}\lambda^m\int\lvert F(\lambda w)\rvert c(\lambda w,x)\delta_Z(w)\Psi(w)\,dw&=\\
(2\pi)^{-m/2}\lambda^m\E\Bigl(\lvert F(\lambda\n)\rvert\Psi(\n)c(\lambda\n,x)\Bigr),
\end{align*}
where $c(v,x)\coloneqq\cos(\langle v,x\rangle+\arg F(v))$ and $\n\sim N(0,I_m).$

We now approximate the above expectation by
\begin{align*}
h(x)=\biggl(\frac{\lambda}{\sqrt{2\pi}}\biggr)^{\!\!m}\E\Bigl(\lvert F(\lambda\n)\rvert\Psi(\n)\Bigl[\lvert\n\rvert>\theta\sqrt{m}\Bigr]c(\lambda\n,x)\Bigr).
\end{align*}
\begin{lem}
\label{lem:step2}
$\lVert g-h\rVert_\infty\leq\frac{\ell R}{\sqrt{\pi m}}V_m\exp(d(K)m\ln 2)\biggl(\frac{R\theta\lambda}{\sqrt{2\pi/e}}\biggr)^{\!\!m}$
\end{lem}
\begin{proof}
Since $\lVert\F\{\diamondsuit\}\rVert_\infty\leq\lVert\diamondsuit\rVert_1$ and $\psi$ is a pdf,
\begin{align*}
\lvert g(x)-h(x)\rvert&\leq\\
\biggl(\frac{\lambda}{\sqrt{2\pi}}\biggr)^{\!\!m}\E\Bigl\lvert\lvert F(\lambda\n)\rvert\Psi(\n)\Bigl[\lvert\n\rvert\leq\theta\sqrt{m}\Bigr]c(\lambda\n,x)\Bigr\rvert&\leq\\
(2\pi)^{-m/2}\lambda^m\lVert\fcirc\rVert_1\P\Bigl\{\lvert\n\rvert\leq\theta\sqrt{m}\Bigr\}.
\end{align*}
The remaining details may be found in the appendix.
\end{proof}

In \S\ref{sec:birdeye}, we said that $h$ would be an ``infinite width'' RVFL with ReLU activation functions and Gaussian inner weights. The following lemma corroborates this.

\begin{lem}
$h=\E(G(\w,\b)\rho(\langle\w,\diamondsuit\rangle+\b))$ on $K,$ where
\begin{itemize}
\item $G(w,b)=-2\sigma R\sqrt{m}\Lambda^2(2\pi)^{-m/2}\lambda^m\lvert F(\Lambda w)\rvert\Psi(w/\sigma)$ $\Bigl[\lvert w\rvert\geq\theta\sigma\sqrt{m}\Bigr]\cos(\Lambda b-\arg F(\Lambda w))$
\item $\b \sim \mathrm{Unif}\Bigl[-\sigma R\sqrt{m},\sigma R\sqrt{m}\Bigr]$
\item $\w\sim N(0,\sigma I_m)$
\end{itemize}
\end{lem}
\begin{proof}
The statement follows from straightforward manipulations of unwieldy integrals, crucially relying on the fact that the ReLU activation function satisfies ``$\ddot{\rho}=\delta_0$''. As always, the details may be found in the appendix.
\end{proof}
We now want to approximate our ``infinite width'' RVFL with a finite width one. Essentially, we are going to leverage Hoeffding's inequality to bound
\begin{align*}
\P\Bigl\{\bigl\lVert\E(H)-\frac{1}{n}\sum_{p=1}^n H_p\bigr\rVert_K>t\Bigr\},
\end{align*}
where $H(x)=G(\w,\b)\rho(\langle\w,x\rangle+\b))$ and $H_1,\ldots,H_n$ are iid copies of $H.$ Because of the $\lVert\cdot\rVert_K$ norm, however, we cannot directly use Hoeffding's inequality. Using Lemma \ref{lem:general_sup_norm_prob_bound} in the appendix, on the other hand, yields that
\begin{lem}
\label{lem:step3}
Let $N_n=\textstyle\frac{1}{n}\sum_{p=1}^n H_p$ and $t>0.$ Then, for large $n,$
\begin{gather*}
\P\Bigl\{\lVert\E(H)-N_n\rVert_K>t\Bigr\}\lessapprox\\
2\lvert K\rvert\Theta_m(2\pi m)^{-m/2}\biggl(\frac{n(t/b)}{R(1+1/\theta)}\biggr)^{\!\!m}\exp(-(n/2)(t/b)^2).
\end{gather*}
In particular, $\Theta_m\approx m^{3/2}$ (see appendix), $n\geq 4m(b/\epsilon)^2$ and
\begin{align*}
b=2R^{m+2}\ell\sqrt{m}(2\pi)^{-m/2}\lambda^{m+1}V_m\exp(d(K)\ln 2)(1+1/\theta).
\end{align*}
\end{lem}
\begin{proof}
A fairly straightforward application of Lemma \ref{lem:general_sup_norm_prob_bound}; see the appendix.
\end{proof}
We are now ready to put Lemmas~\ref{lem:step1}, \ref{lem:step2} and \ref{lem:step3} together to obtain the desideratum. By setting
\begin{gather*}
\alpha=\frac{m}{m^2+3m+1};\\
\beta=\frac{1}{m^2+3m+1};\\
\Lambda=\frac{1}{\sigma}\times\frac{\lcirc}{\epsilon\alpha}\Bigl(2-2^{1/3}am^{-2/3}\Bigr)\sqrt{m};\\
\frac{1}{\theta}=\biggl(\frac{1}{\epsilon\beta}\times\frac{\ell R}{\sqrt{\pi m}}V_m\biggr)^{\mathrlap{\!\!1/m}}\enspace\frac{R\lambda}{\sqrt{2\pi/e}}\exp(d(K)\ln 2),
\end{gather*}
we obtain that $\lVert f-g\rVert_\infty\leq\alpha\epsilon$ and $\lVert g-h\rVert_\infty\leq\beta\epsilon$ by design. As such, letting $\gamma=1-\alpha-\beta$ yields that
\begin{gather}
\P\Bigl\{\lVert f-N_n\rVert_K>\epsilon\Bigr\}=\P\Bigl\{\lVert\fcirc-N_n\rVert_K>\epsilon\Bigr\}\notag\\
\leq\P\Bigl\{\lVert\fcirc-g\rVert_\infty+\lVert g-h\rVert_\infty+\lVert h-N_n\rVert_K>\epsilon\Bigr\}\notag\\
\leq\P\Bigl\{\alpha\epsilon+\beta\epsilon+\lVert h-N_n\rVert_K>\epsilon\Bigr\}=\P\Bigl\{\lVert h-N_n\rVert_K>\gamma\epsilon\Bigr\}\notag\\
\leq 2\lvert K\rvert\Theta_m(2\pi m)^{-m/2}\biggl(\frac{n(\gamma\epsilon/b)}{R(1+1/\theta)}\biggr)^{\!\!m}\exp(-(n/2)(\gamma\epsilon/b)^2)\label{eq:bound_before_decreasing}
\end{gather}
where $b=2R^{m+2}\ell\sqrt{m}(2\pi)^{-m/2}\lambda^{m+1}V_m\exp(d(K)\ln 2)(1+1/\theta).$

As $n^m\exp(-(n/2)(\gamma\epsilon/b)^2)$ is decreasing in $n$ once $n\geq 2m(\gamma\epsilon/b)^2$ we may plug in \eqref{eq:n_bound} in place of $n$ and simplify a great deal (see the appendix for some details) to obtain the desideratum.

Lastly, note that $N_n$ is indeed an RVFL whose outer weights are iid copies of $G(\w,\b)/n.$

\section*{Acknowledgements}
Palina Salanevich is partially supported by NWO Talent programme Veni ENW grant, file number VI.Veni.212.176.

\appendix

\section{Proofs of the lemmas}

Here you may find parts of proofs omitted in the main article. The numbering herein employed is the same as in the main article.

We remind the reader that $K\subset\R^m$ is a convex body, $f:K\to\R$ is $\ell$-Lipschitz continuous, and
\begin{gather*}
g(x)=\prefac\int F(v)\exp(i\langle v,x\rangle-\lvert v\rvert^2/2\lambda^2)\Psi(v/\lambda)\,dv;\\
h(x)=(2\pi)^{-m/2}\lambda^m\E\Bigl(\lvert F(\lambda\n)\rvert\Psi(\n)\Bigl[\lvert\n\rvert>\theta\Bigr]c(\lambda\n,x)\Bigr).
\end{gather*}
Furthermore, $H_1,\ldots,H_p$ are iid copies of $G(\w,\b)\rho(\langle\w,\diamondsuit\rangle+\b))$ defined in Lemma~4 of the main paper.
\begin{lem}
If $U\subset\R^m$ is bounded and $f:U\to\R$ is $\ell$-Lipschitz continuous, $f$ can \\ be extended to a compactly supported $\lcirc$-Lipschitz continuous function $\fcirc.$
\end{lem}
\begin{proof}
It is well-known that $f$ may be $\ell$-Lipschitz extended to $\ovl{U}$ uniquely, so we assume WLOG that $U$ is closed (or equivalently, compact).

Let $\fcirc_+$ be as in \eqref{eq:fcirc+}. It has long since been known that
\begin{align*}
\max_{u\in U}\Bigl(\lvert f(u)\rvert-\ell\lvert x-u\rvert\Bigr)
\end{align*}
is an $\ell$-Lipschitz extension of $\lvert f\rvert$ \cite[Thm 1]{mcshane}. As $U\ni u\to\lvert f(u)\rvert-\ell\lvert x-u\rvert$ is continuous and $U$ is compact, the maximum is attained. Since $\rho$ is the identity function on $[0,\infty)$ and 1-Lipschitz continuous, $\fcirc_+$ is a compactly supported $\ell$- Lipschitz extension of $\lvert f\rvert;$ compactly supported because $U$ is bounded, so
\begin{align*}
\{\fcirc_+>0\}=\bigcup_{u\in U}B_u\text{ is bounded,}
\end{align*}
where $B_u=\{x\in\R^m:\lvert x-u\rvert<\tfrac{1}{\ell}\lvert f(u)\rvert\}.$ We now extend $\sg f.$

Let $u,v\in U.$ If $B_u\cap B_v\neq\varnothing,$ then $\sg f(u)=\sg f(v);$ otherwise
\begin{align*}
\lvert u-v\rvert\leq\lvert u-x\rvert+\lvert x-v\rvert<\Bigl(\lvert f(u)\rvert+\lvert f(v)\rvert\Bigr)/\ell=\tfrac{1}{\ell}\lvert f(u)-f(v)\rvert
\end{align*}
for any $x\in B_u\cap B_v.$ We may thus extend $\sg f$ as follows:
\begin{align}\label{eq:sg_extension}
\R^m\ni x\mapsto
\begin{dcases*}
\hphantom{\pm}0 & if $\fcirc_+(x)=0,$ \\
\pm1 & if $x\in B_u$ and $\sg f(u)=\pm1.$
\end{dcases*}
\end{align}
Multiplying $\fcirc_+$ by \eqref{eq:sg_extension} yields the desired $\fcirc.$

Let $u,v\in U$ and $x,z\in\R^m$ be such that $\fcirc_+(x)=\lvert f(u)\rvert-\ell\lvert x-u\rvert>0$ (so $x\in B_u)$ and analogously $\fcirc_+(z)=\lvert f(v)\rvert-\ell\lvert z-v\rvert>0$ (so $z\in B_v).$ We verify that
\begin{align*}
\sg\fcirc(x)\neq\sg\fcirc(z)\Rightarrow\lvert\fcirc(x)-\fcirc(z)\rvert\leq\ell\lvert x-z\rvert;
\end{align*}
indeed, $\sg f(u)=\sg f(v)\Rightarrow \sg\fcirc(x)=\sg\fcirc(z),$ so $\sg f(u)\neq\sg f(v)$ and hence
\begin{align*}
\lvert\fcirc(x)-\fcirc(z)\rvert&=\fcirc_+(x)+\fcirc_+(z)=\lvert f(u)\rvert-\ell\lvert x-u\rvert+\lvert f(v)\rvert-\ell\lvert z-v\rvert\\
&=\lvert f(u)-f(v)\rvert-\ell\Bigl(\lvert x-u\rvert+\lvert z-v\rvert\Bigr)\\
&\leq\ell\Bigl(\lvert u-v\rvert-\lvert x-u\rvert-\lvert z-v\rvert\Bigr)\leq\ell\lvert x-z\rvert
\end{align*}
by the reverse triangle inequality.
\end{proof}

\begin{lem}
$\lVert\fcirc-g\rVert_\infty\leq\frac{\lcirc}{\lambda}\Bigl(2-2^{1/3}am^{-2/3}\Bigr)\sqrt{m}.$
\end{lem}
\begin{proof}
Using various substitutions yields that
\begin{align*}
\prefac\int F(v)\exp(i\langle v,x\rangle-\lvert v\rvert^2/2\lambda^2)\Psi(v/\lambda)\,dv&=\\
\prefac\iint\fcirc(u)\exp(-i\langle v,u\rangle)\exp(i\langle v,x\rangle-\lvert v\rvert^2/2\lambda^2)\Psi(v/\lambda)\,du\,dv&=\\
\prefac\iint\fcirc(u)\exp(i\langle v,x-u\rangle-\lvert v\rvert^2/2\lambda^2)\Psi(v/\lambda)\,dv\,du&=\\
\prefac\int\fcirc(x-t)\int\exp(i\langle v,t\rangle-\lvert v\rvert^2/2\lambda^2)\Psi(v/\lambda)\,dv\,dt&=\\
\prefac\int\fcirc(x-s/\lambda)\int\exp(i\langle w,s\rangle-\lvert w\rvert^2/2)\Psi(w)\,dw\,ds&=\\
\int\fcirc(x-s/\lambda)(\delta_Z*\psi)(s)\,ds,
\end{align*}
where the second equality follows from Fubini's theorem because $\fcirc$ and $\Psi$ are compactly supported. In the main article it was shown that
\begin{align*}
\lVert\fcirc-g\rVert_\infty\leq\frac{\lcirc}{\lambda}\int\lvert s\rvert(\delta_Z*\psi)(s)\,ds,
\end{align*}
so all that remains to be shown is that $\int\lvert s\rvert(\delta_Z*\psi)(s)\,ds\leq\Bigl(2-2^{1/3}am^{-2/3}\Bigr)\sqrt{m}.$

Since $\psi$ is a pdf, $\psi=\delta_X$ for some random variable $X.$ Thus,
\begin{align*}
\int\lvert s\rvert(\delta_Z*\psi)(s)\,ds=\int\lvert s\rvert(\delta_Z*\delta_X)(s)\,ds=\E\lvert Z+X\rvert\leq\E\lvert Z\rvert+\E\lvert X\rvert.
\end{align*}
Since $\lvert Z\rvert$ is chi distributed with $m$ degrees of freedom, Wendel's inequality \cite{wendel} yields that
\begin{align*}
\E\lvert Z\rvert=\frac{\Gamma((m+1)/2)}{\Gamma(m/2)}\sqrt{2}\leq\sqrt{m}.
\end{align*}
By Jensen's inequality and \cite[thm 5.1]{ehm},
\begin{align*}
\E\lvert X\rvert\leq\sqrt{\displaystyle\int\lvert x\rvert^2\psi(x)\,dx}=2j_\nu/\sqrt{m}.
\end{align*}
Indeed, as follows from the scaling property of the Fourier transform,
\begin{align*}
\int\lvert x\rvert^2\F^{-1}\Bigl\{(\omega*\omega)\Bigl(\diamondsuit/\sqrt{m}\Bigr)\Bigr\}(x)\,dx&=\\
\Bigl(\!\sqrt{m}\Bigr)^{\!m}\int\lvert x\rvert^2\F^{-1}\{\omega*\omega\}\Bigl(x\sqrt{m}\Bigr)\,dx&=\\
\frac{1}{m}\int\lvert u\rvert^2\F^{-1}\{\omega*\omega\}(u)\,du&=4j_\nu^2/m.
\end{align*}
We now claim that $2j_\nu/\sqrt{m}<\sqrt{m}-2^{1/3}am^{-1/6}.$ In \cite{qu}, it was shown that
\begin{align*}
j_\nu<\nu-a(\nu/2)^{1/3}+\tfrac{3}{20}a^2(\nu/2)^{-1/3}
\end{align*}
for all $\nu>0.$ If $m\geq3,$ then $\nu=m/2-1>0,$ so
\begin{align*}
2j_\nu/\sqrt{m}&<\sqrt{m}-2/\sqrt{m}-2^{2/3}a(m/2-1)^{1/3}/\sqrt{m}+\tfrac{6}{20}a^2(\nu/2)^{-1/3}/\sqrt{m}\\
&<\sqrt{m}-2^{2/3}a(m/2)^{1/3}/\sqrt{m}+(\tfrac{6}{20}a^2(\nu/2)^{-1/3}-2)/\sqrt{m}\\
&=\sqrt{m}-2^{1/3}am^{-1/6}+(\tfrac{6}{20}a^2(\nu/2)^{-1/3}-2)/\sqrt{m}.
\end{align*}
Since $(\nu/2)^{-1/3}$ is decreasing in $m$ and $\tfrac{6}{20}a^2(\nu/2)^{-1/3}<2$ for $m=3$ (verified numerically), the last term is negative for all $m\geq 3$ and may thus be discarded to yield our claim for all $m\geq 3.$ The validity of the claim for $m=1,2$ was directly numerically verified. Putting everything together yields the desideratum.
\end{proof}

\begin{lem}
$\lVert g-h\rVert_\infty\leq\frac{\ell R}{\sqrt{\pi m}}V_m\exp(d(K)m\ln 2)\biggl(\frac{R\theta\lambda}{\sqrt{2\pi/e}}\biggr)^{\!\!m}$
\end{lem}
\begin{proof}
In the main article we have already shown that
\begin{align*}
\lVert g-h\rVert_\infty\leq(2\pi)^{-m/2}\lambda^m\lVert\fcirc\rVert_1\P\Bigl\{\lvert\n\rvert\leq\theta\sqrt{m}\Bigr\},
\end{align*}
so all that remains is to bound $\lVert\fcirc\rVert_1$ and $(2\pi)^{-m/2}\lambda^m\P\Bigl\{\lvert\n\rvert\leq\theta\sqrt{m}\Bigr\}.$

By construction $\lVert\fcirc\rVert_\infty=\lVert f\rVert_\infty=M$ and
\begin{align*}
\kcirc\subset K+\{x\in\R^m:\lvert x\rvert\leq\lVert f\rVert_\infty/\ell\}\subset K+\{x\in\R^m:\lvert x\rvert\leq R\},
\end{align*}
because $M\leq\ell R;$ indeed, $K$ having circumradius $R$ means $\diam(K)=2R$ and so $2M=\max f-\min f\leq\ell\diam(K)=2\ell R$ because $f$ is $\ell$-Lipschitz. Ergo $\lVert\fcirc\rVert_1\leq$
\begin{align*}
\lVert\fcirc\rVert_\infty\lvert\kcirc\rvert\leq MV_mR^m\times\frac{\displaystyle\lvert K+\{x\in\R^m:\lvert x\rvert\leq R\}\rvert}{\displaystyle\lvert\{x\in\R^m:\lvert x\rvert\leq R\}\rvert}\leq\ell RV_mR^m\exp(d(K)\ln 2).
\end{align*}
Concerning the latter, $\lvert\n\rvert$ has a chi distribution with $m$ degrees of freedom, so the cdf of $\lvert\n\rvert$ may be expressed as $P(m/2,\diamondsuit^2/2)$ \cite[{\S}8.2(i)]{dlmf}. If $x\geq 0,$ then \cite[(8.6.3)]{dlmf}
\begin{align*}
P(a,x)=\frac{\displaystyle x^a}{\Gamma(a)}\int_0^\infty\exp(-at-xe^{-t})\,dt\leq\frac{\displaystyle x^a}{\Gamma(a)}\int_0^\infty\exp(-at)\,dt=\frac{\displaystyle x^a}{a\Gamma(a)}
\end{align*}
Additionally $a\Gamma(a)>(a/e)^a\sqrt{2\pi a}$ for all $a>0$~\cite[(5.6.1)]{dlmf}, so
\begin{align*}
(2\pi)^{-m/2}\lambda^m\P\Bigl\{\lvert\n\rvert\leq\theta\sqrt{m}\Bigr\}&=(2\pi)^{-m/2}\lambda^m P(m/2,m\theta^2/2)\\
&<(2\pi)^{-m/2}\lambda^m\frac{1}{\sqrt{\pi m}}\biggl(\frac{\displaystyle m\theta^2/2}{m/2e}\biggr)^{\!\!m/2}\\
&=\frac{1}{\sqrt{\pi m}}\biggl(\frac{\theta\lambda}{\sqrt{2\pi/e}}\biggr)^{\!\!m}
\end{align*}
Multiplying the obtained bounds yields the desideratum.
\end{proof}

\begin{lem}
$h=\E(G(\w,\b)\rho(\langle\w,\diamondsuit\rangle+\b))$ on $K,$ where
\begin{itemize}
\item $G(w,b)=-2\sigma R\sqrt{m}\Lambda^2(2\pi)^{-m/2}\lambda^m\lvert F(\Lambda w)\rvert\Psi(w/\sigma)$ $\Bigl[\lvert w\rvert\geq\theta\sigma\sqrt{m}\Bigr]\cos(\Lambda b-\arg F(\Lambda w))$
\item $\b$ being uniformly distributed on $\Bigl[-\sigma R\sqrt{m},\sigma R\sqrt{m}\Bigr]$
\item $\w\sim N(0,\sigma I_m)$
\end{itemize}
\end{lem}
\begin{proof}
Letting $N\sim N(0,\sigma I_m)$ and $\phi(v)=\arg F(v)$ allows us to write
\begin{align*}
(2\pi)^{-m/2}\lambda^m\E\Bigl(\lvert F(\lambda\n)\rvert\Psi(\n)\Bigl[\lvert\n\rvert\geq\theta\sqrt{m}\Bigr]c(\n,x)\Bigr)&=\\
(2\pi)^{-m/2}\lambda^m\int\lvert F(\lambda u)\rvert\Psi(u)\Bigl[\lvert u\rvert\geq\theta\sqrt{m}\Bigr]c(\lambda u,x)\delta_Z(u)\,du&=\\
(2\pi)^{-m/2}\lambda^m\int\lvert F(\Lambda w)\rvert\Psi(w/\sigma)\Bigl[\lvert w\rvert\geq\theta\sigma\sqrt{m}\Bigr]c(\Lambda w,x)\delta_N(w)\,dw&=\\
\int\Bigl((2\pi)^{-m/2}\lambda^m\lvert F(\Lambda w)\rvert\Bigl[\lvert w\rvert\geq\theta\sigma\sqrt{m}\Bigr]\Psi(w/\sigma)\Bigr)\\
\cos(\phi(\Lambda w)+\Lambda\langle w,x\rangle)\Bigl[\lvert\langle w,x\rangle\rvert\leq\sigma R\sqrt{m}\Bigr]\delta_N(w)\,dw,
\end{align*}
since $\lvert\langle w,x\rangle \rvert\leq\lvert w\rvert.\lvert x\rvert$ and the support of $\Psi$ is $\Bigl\{w\in\R^m:\lvert w\rvert\leq\sqrt{m}\Bigr\}.$

We also used that $x\in K\Rightarrow\lvert x\rvert\leq R$ because $K$ has circumradius $R.$

Now, using the fundamental theorem of calculus and integration by parts,
\begin{align*}
\cos(\phi+\Lambda z)\Bigl[\vert z\rvert\leq B\Bigr]&=-\Lambda\int_{-\infty}^z\sin(\phi+\Lambda y)\Bigl[\lvert y\rvert\leq B\Bigr]\,dy\\
&=-\Lambda\int\sin(\phi+\Lambda y)\Bigl[\lvert y\rvert\leq B\Bigr]\dot{\rho}(z-y)\,dy\\
&=-\Lambda^2\int\cos(\phi+\Lambda y)\Bigl[\lvert y\rvert\leq B\Bigr]\rho(z-y)\,dy\\
&=-\Lambda^2\int\cos(\phi-\Lambda b)\rho(z+b)\Bigl[\lvert b\rvert\leq B\Bigr]\,db.
\end{align*}
Upon plugging back in, this yields that
\begin{align*}
h(x)=\iint G(w,b)\rho(\langle w,x\rangle+b)\frac{1}{2\sigma R\sqrt{m}}\Bigl[\lvert b\rvert\leq\sigma R\sqrt{m}\Bigr]\delta_N(w)\,db\,dw,
\end{align*}
which is the expanded form of desired expectation.
\end{proof}

\begin{lem}
Let $N_n=\textstyle\frac{1}{n}\sum_{p=1}^n H_p$ and $t>0.$ Then, for large $n,$
\begin{gather*}
\P\Bigl\{\lVert\E(H)-N_n\rVert_K>t\Bigr\}\lessapprox\\
2\lvert K\rvert\Theta_m(2\pi m)^{-m/2}\biggl(\frac{n(t/b)}{R(1+1/\theta)}\biggr)^{\!\!m}\exp(-(n/2)(t/b)^2).
\end{gather*}
In particular, $\Theta_m\approx m^{3/2}$ (see below), $n\geq 4m(b/\epsilon)^2$ and
\begin{align*}
b=2R^{m+2}\ell\sqrt{m}(2\pi)^{-m/2}\lambda^{m+1}V_m\exp(d(K)\ln 2)(1+1/\theta).
\end{align*}
\end{lem}
\begin{proof}
To prove the desired bound, we rely on Lemma 6 stated and proven below. To apply Lemma \ref{lem:general_sup_norm_prob_bound} we need to bound $H(x)$ uniformly in $x\in K$ and the Lipschitz constant of $H(\omega,\cdot)$ uniformly in $\omega\in\Omega.$ We begin with the latter.

Since $\rho$ is 1-Lipschitz continuous,
\begin{align*}
\lvert H(x)-H(z)\rvert\leq\lvert G(\w,\b)\rvert.\lvert\langle\w,x-z\rangle\rvert\leq\lvert G(\w,\b)\rvert.\lvert\w\rvert.\lvert x-z\rvert
\end{align*}
so the Lipschitz constant of $H$ is bounded by $\lvert G(\w,\b)\rvert.\lvert\w\rvert\leq$
\begin{gather*}
2\sigma R\sqrt{m}\Lambda^2(2\pi)^{-m/2}\lambda^m\lvert F(\Lambda\w)\rvert.\lvert\Psi(\w/\sigma)\rvert.\lvert\w\rvert\\
\leq 2R\sqrt{m}(2\pi)^{-m/2}\lambda^{m+1}\lvert\Lambda\w\rvert.\lvert F(\Lambda\w)\rvert.\lVert\Psi\rVert_\infty\\
\leq 2R\sqrt{m}(2\pi)^{-m/2}\lambda^{m+1}\lVert\Psi\rVert_\infty\Bigl\lVert\lvert\Lambda\diamondsuit\rvert.\lvert F(\Lambda\diamondsuit)\rvert\Bigr\rVert_\infty\\
=2R\sqrt{m}(2\pi)^{-m/2}\lambda^{m+1}\lVert\Psi\rVert_\infty\Bigl\lVert\lvert\diamondsuit\rvert.\lvert F(\diamondsuit)\rvert\Bigr\rVert_\infty
\end{gather*}
Now, using the fact that $\lVert\F\{\diamondsuit\}\rVert_\infty\leq\lVert\diamondsuit\rVert_1$ and utilizing Minkowski's integral inequality twice yields that $\lVert\Psi\rVert_\infty\leq\lVert\psi\rVert_1=1$ because $\psi$ is a pdf, and
\begin{align}
\Bigl\lVert\lvert\diamondsuit\rvert.\lvert F(\diamondsuit)\rvert\Bigr\rVert_\infty&=\Bigl\lVert\lvert\diamondsuit F(\diamondsuit)\rvert\Bigr\rVert_\infty\leq\lvert\{\lVert F(\diamondsuit)\diamondsuit_p\rVert_\infty\}_{p=1}^m\rvert\notag\\
&=\lvert\{\lVert\F\{\partial_p\fcirc\}\rVert_\infty\}_{p=1}^m\rvert\leq\lvert\{\lVert\partial_p\fcirc\rVert_1\}_{p=1}^m\rvert\notag\\
&\leq\Bigl\lVert\lvert\{\partial_p\fcirc\}_{p=1}^m\rvert\Bigr\rVert_1=\lVert\nabla\fcirc\rVert_1\leq\lcirc\lvert\kcirc\rvert,\label{eq:gradient}
\end{align}
where $\lvert\{\diamondsuit_p\}_{p=1}^m\rvert$ denotes the Euclidean norm of the vector $\{\diamondsuit_p\}_{p=1}^m.$ The first equality follows since $F(\diamondsuit)\in\mathbb{C}$ for each $\diamondsuit\in\R^m$ and the last inequality holds because $\fcirc,$ and so $\nabla\fcirc,$ is 0 outside of $\kcirc.$

Previously we have seen that $\lvert\kcirc\rvert\leq V_mR^m\exp(d(K)\ln 2),$ hence the Lipschitz constant of $H$ is bounded by
\begin{align*}
2R^{m+1}\ell\sqrt{m}(2\pi)^{-m/2}\lambda^{m+1}V_m\exp(d(K)\ln 2)\eqcolon k
\end{align*}
In \eqref{eq:gradient} we tacitly used Rademacher's theorem to conclude that $\nabla\fcirc$ exists almost everywhere. We also used that $\lVert\nabla\fcirc\rVert_\infty\leq\ell,$ which can be seen as follows:

It suffices to show that $\lvert\nabla\fcirc(x)\rvert\leq\ell$ for all $x\in\kcirc$ for which $\nabla\fcirc(x)$ exists, since $\fcirc,$ and therefore $\nabla\fcirc,$ is 0 outside of $\kcirc,$ so suppose $x\in\kcirc$ is such that $\nabla\fcirc(x)$ exists. There is nothing to prove if $\nabla\fcirc(x)=0,$ so assume that $\nabla\fcirc(x)\neq 0$ as well. Then $x\in\kcirc$ and $\nabla\fcirc(x)\neq 0,$ so the reverse triangle inequality yields that
\begin{gather*}
\lim_{h\to 0}\biggl\lvert\frac{\lvert\fcirc(x+hu)-\fcirc(x)\rvert}{h}-\lvert\nabla\fcirc(x)\rvert\biggr\rvert\leq\lim_{h\to 0}\biggl\lvert\frac{\fcirc(x+hu)-\fcirc(x)}{h}-\lvert\nabla\fcirc(x)\rvert\biggr\rvert\\
=\lim_{h\to 0}\frac{\lvert\fcirc(x+hu)-\fcirc(x)-\langle hu,\nabla\fcirc(x)\rangle\rvert}{\lvert hu\rvert}=0,
\end{gather*}
where $u=\frac{\nabla\fcirc(x)}{\lvert\nabla\fcirc(x)\rvert}$ and $h>0.$ As such,
\begin{align*}
\lvert\nabla\fcirc(x)\rvert=\lim_{h\to 0}\frac{\lvert\fcirc(x+hu)-\fcirc(x)\rvert}{h}\leq\ell,
\end{align*}
because $u$ is a unit vector.

Digression finished, we now bound $H(x)$ uniformly in $x\in K.$
\begin{gather*}
\lvert H(x)\rvert=\lvert G(\w,\b)\rho(\langle\w,x\rangle+\b)\rvert\\
\leq 2\sigma R\sqrt{m}\Lambda^2(2\pi)^{-m/2}\lambda^m\lvert F(\Lambda\w)\rvert.\lvert\Psi(\w/\sigma)\rvert\Bigl(\langle\w,x\rangle+\sigma R\sqrt{m}\Bigr)\Bigl[\lvert\w\rvert\geq\theta\sigma\sqrt{m}\Bigr]\\
\leq 2\sigma R\sqrt{m}\Lambda(2\pi)^{-m/2}\lambda^m\Bigl(\Lambda\lvert F(\Lambda\w)\rvert\Bigr)R(1+1/\theta)\lvert\w\rvert\leq kR(1+1/\theta)\eqcolon b.
\end{gather*}
Applying Lemma \ref{lem:general_sup_norm_prob_bound} and simplifying a little bit yields the desideratum.
\end{proof}

\subsection{Lemma \ref{lem:general_sup_norm_prob_bound}}

Before stating and proving Lemma \ref{lem:general_sup_norm_prob_bound}, we need some preparation. Beginning with some notation, let $B(\delta)=\{x:\lvert x\rvert<\delta\}$ and let $N_S^\delta$ be the $\delta$-covering number of $S\subset\R^m$ by Euclidean balls \cite[Def.\ 4.2.2]{vershynin}. A corollary of \cite[Thm 7.1.1]{boroczky} is that
\begin{prop}\label{prop:boroczky}
Let $K\subset\R^m$ be a convex body. For all $\delta\in(0,\rho),$ where $\rho$ is the maximal radius of any ball contained in $K,$
\begin{align*}
\theta_m\lvert K\rvert(1-\delta/\rho)^m\leq N_K^\delta\lvert B(\delta)\rvert\leq\theta_m\lvert K\rvert(1+\delta/\rho)^m
\end{align*}
and $\theta_m$ is a constant depending on $m$ only. Ergo, $N_K^\delta\lvert B(\delta)\rvert\approx\theta_m\lvert K\rvert$ for small $\delta.$
\end{prop}
It is known that $m\preccurlyeq\theta_m\preccurlyeq m\ln m$ with reasonable constants \cite[Thm 8.2.1]{boroczky}.
\begin{lem}\label{lem:general_sup_norm_prob_bound}
Let $K\subset\R^m$ be a convex body and let $(\Omega,\Sigma,\P)$ be a probability space that accommodates $n$ iid copies of $H,$ denoted $\{H_p\}_{p=1}^n,$ where $H:\Omega\times K\to\mathbb{C}$ is such that $H(\omega,\cdot)$ is $k$-Lipschitz continuous for almost every $\omega\in\Omega$ and $H(x)$ is a random variable bounded by $b>0$ for every $x\in K.$ While $n\geq 4m(b/\epsilon)^2$ for some $\epsilon>0,$ we have that
\begin{align*}
\P\Bigl\{\bigl\lVert\mathds{E}(H)-\frac{1}{n}\sum_{p=1}^n H_p\bigr\rVert_K>\epsilon\Bigr\}\leq 2N_K^\delta\delta^m(k\sqrt{e}/m)^m\zeta^m e^{-(\epsilon/4)\zeta}
\end{align*}
where $\delta,\zeta$ are defined by
\begin{align}
k\delta=\frac{m}{\zeta}\text{ with }\zeta=(\epsilon n/b^2)\biggl(1+\sqrt{\displaystyle 1-4m(b/\epsilon)^2/n}\biggr)\label{eq:kdelta}
\end{align}
Note that $\zeta\approx 2\epsilon n/b^2$ for large $n,$ so in this r{\'e}gime we have that
\begin{align*}
\P\Bigl\{\lVert\mathds{E}(H)-\frac{1}{n}\sum_{p=1}^n H_p\bigr\rVert_K>\epsilon\Bigr\}\lessapprox 2\lvert K\rvert\Theta_m(2\pi m)^{-m/2}(\epsilon kn/b^2)^m\exp(-(n/2)(\epsilon/b)^2)
\end{align*}
where $\Theta_m\coloneq\theta_m\sqrt{\pi}(m^3+m^2+\tfrac{1}{2}m+\tfrac{1}{30})^{1/6}$
\end{lem}
\begin{rem}
By the triangle inequality, $\mathds{E}(H)$ is also $k$-Lipschitz, so $K$ being separable (as a subset of a separable metric space) yields that
\begin{center}
$\bigl\lVert\mathds{E}(H)-\frac{1}{n}\sum_{p=1}^n H_p\bigr\rVert_K$ is really a countable supremum and thus measurable.
\end{center}
\end{rem}
Let $B_z^\delta$ be the (separable) metric ball in the metric space $K$ with centre $z$ and radius $\delta.$
\begin{proof}
Let $N$ be a minimal $\delta$-net of $K$ with corresponding covering number $N_K^\delta.$ Then
\begin{align*}
\P\Bigl\{\bigl\lVert\mathds{E}(H)-\frac{1}{n}\sum_{p=1}^n H_p\bigr\rVert_K>\epsilon\Bigr\}\leq\sum_{z\in N}\P\Bigl\{\bigl\lVert\mathds{E}(H)-\frac{1}{n}\sum_{p=1}^n H_p\bigr\rVert_{B_z^\delta}>\epsilon\Bigr\}
\end{align*}
Since $\mathds{E}(H)-\frac{1}{n}\sum_{p=1}^n H_p$ is $2k$-Lipschitz,
\begin{align*}
\bigl\lvert\mathds{E}(H(z))-\frac{1}{n}\sum_{p=1}^n H_p(z)\bigr\rvert\geq\bigl\lvert\mathds{E}(H(x))-\frac{1}{n}\sum_{p=1}^n H_p(x)\bigr\rvert-2k\delta
\end{align*}
for all $x\in B_z^\delta$ and $z\in N,$ so for all $z\in N,$
\begin{align*}
\P\Bigl\{\bigl\lVert\mathds{E}(H)-\frac{1}{n}\sum_{p=1}^n H_p\bigr\rVert_{B_z^\delta}>\epsilon\Bigr\}\leq\P\Bigl\{\bigl\lvert\mathds{E}(H(z))-\frac{1}{n}\sum_{p=1}^n H_p(z)\bigr\rvert>\epsilon-2k\delta\Bigr\}
\end{align*}
Since $n\geq 4m(b/\epsilon)^2\Rightarrow k\delta\leq\epsilon/4<\epsilon/2,$ Hoeffding's inequality yields that
\begin{align*}
\P\Bigl\{\bigl\lvert\mathds{E}(H(z))-\frac{1}{n}\sum_{p=1}^n H_p(z)\bigr\rvert>\epsilon-2k\delta\Bigr\}\leq 2\exp(-\tfrac{1}{2}(n/b^2)(\epsilon-2k\delta)^2)
\end{align*}
so all in all we have that
\begin{align*}
\P\Bigl\{\bigl\lVert\mathds{E}(H)-\frac{1}{n}\sum_{p=1}^n H_p\bigr\rVert_K>\epsilon\Bigr\}\leq 2N_K^\delta(k\delta)^m(k\delta)^{-m}\exp(-2(n/b^2)(\epsilon/2-k\delta)^2)
\end{align*}
One can verify that \eqref{eq:kdelta} locally minimizes $(k\delta)^{-m}\exp(-2(n/b^2)(\epsilon/2-k\delta)^2)$ at
\begin{align*}
(\sqrt{e}/m)^m\zeta^m e^{-(\epsilon/4)\zeta}
\end{align*}
which readily yields that $\P\Bigl\{\bigl\lVert\mathds{E}(H)-\frac{1}{n}\sum_{p=1}^n H_p\bigr\rVert_K>\epsilon\Bigr\}\lessapprox$
\begin{align*}
2N_K^\delta\lvert B(\delta)\rvert\frac{\displaystyle(\sqrt{e}/m)^m}{\lvert B(1)\rvert}(\epsilon kn/b^2)^m\exp(-(n/2)(\epsilon/b)^2)
\end{align*}
Now, from \cite{alzer2003ramanujan} we know that
\begin{align*}
\frac{1}{\lvert B(1)\rvert}=\pi^{-m/2}\Gamma(\tfrac{m}{2}+1)<\sqrt{\pi}(m^3+m^2+\tfrac{1}{2}m+\tfrac{1}{30})^{1/6}\biggl(\frac{m}{2\pi e}\biggr)^{\!\!m/2}
\end{align*}
so, by Proposition \ref{prop:boroczky}, for large $n,$
\begin{align*}
N_K^\delta\lvert B(\delta)\rvert\frac{\displaystyle(\sqrt{e}/m)^m}{\lvert K(1)\rvert}\lessapprox\lvert K\rvert\Theta_m(2\pi m)^{-m/2}\tag*{\qedhere}
\end{align*}
\end{proof}

\subsection{Main Theorem simplifications}

We will show that $2m(k+m+2)\ln(1/\epsilon)(b/\gamma\epsilon)^2\leq\eqref{eq:n_bound}.$ 

Since $(k+m+2)\ln(1/\epsilon)\geq 2>1,$ we may indeed plug \eqref{eq:n_bound} into \eqref{eq:bound_before_decreasing} at the cost of increasing the bound. We trust the reader to verify that
\begin{gather}
2m(k+m+2)\ln(1/\epsilon)(b/\gamma\epsilon)^2=\notag\\
(k+m+2)(\ell R/\epsilon)^{2(m+3+1/m)}\exp(d(K)4\ln 2)\ln(1/\epsilon)(1+\theta)^2\notag\\
\begin{gathered}\label{eq:before_bound}
e\pi(\pi m)^{-1/m}\biggl(4V_m^{1+1/m}\biggl(1+\frac{m+1}{m(m+2)}\biggr)^{\!\!m+2}\\
(m^2+3m+1)^{1+1/m}\Bigl((2-2^{1/3}am^{-2/3})\sqrt{m/2\pi}\Bigr)^{\!m+2}\biggr)^{\!\!2}
\end{gathered}
\end{gather}
Comparing with \eqref{eq:n_bound}, it thus suffices to prove that
\begin{align*}
\eqref{eq:before_bound}\leq\frac{8}{\pi}\biggl(\frac{5}{\pi}(2-2^{1/3}am^{-2/3})^2(m^2+3m+1)\exp(3-2^{-2/3}am^{1/3})\biggr)^{\!\!2}(4e)^m
\end{align*}
It follows from (4.5.13), (5.6.1) and (5.19.4) in \cite{dlmf} that
\begin{align*}
e\pi V_m^{2+2/m}&\leq\frac{1}{2\pi}(\pi m)^{-1/m}\biggl(\frac{2e\pi}{m}\biggr)^{\!\!m}\\
\biggl(1+\frac{m+1}{m(m+2)}\biggr)^{\!\!m+2}&\leq\exp(1+1/m)\\
(2-2^{1/3}am^{-2/3})^m&\leq 2^m\exp(-2^{-2/3}am^{1/3})
\end{align*}
Applying all three inequalities above and tidying up a bit yields that
\begin{gather*}
\eqref{eq:before_bound}\leq\frac{8}{\pi}\biggl(\biggl(\frac{m^2+3m+1}{\pi m/e}\biggr)^{\!\!1/m}(2-2^{1/3}am^{-2/3})^2\\
(m^2+3m+1)\exp(2-2^{-2/3}am^{1/3})\biggr)^{\!\!2}(4e)^m
\end{gather*}
so we are done if $(0,\infty)\ni s\mapsto\biggl(\frac{s^2+3s+1}{\pi s/e}\biggr)^{\!\!1/s}$ is decreasing (since $\mathbb{N}\ni m).$

Since $\ln\diamondsuit$ is increasing, we may equivalently check whether
\begin{align}
(0,\infty)\ni s\mapsto\frac{1}{s}\Bigl(1+\ln(s^2+3s+1)-\ln(\pi s)\Bigr)\label{eq:ln_decreasing}
\end{align}
is decreasing. Differentiating \eqref{eq:ln_decreasing} with respect to $s,$ this is tantamount to
\begin{align*}
s^{-2}\biggl(\frac{s(3s+2)}{s^2+3s+1}-2-\ln(s^2+3s+1)+\ln(\pi s)\biggr)<0
\end{align*}
for all $s>0.$ Since $\frac{s(3s+2)}{s^2+3s+1}-2=-\frac{3s+2}{s^2+3m+1}<0$ for all $s>0,$
\begin{align*}
\ln(\pi s)<\ln(s^2+3s+1)\Leftrightarrow s^2+3s+1>\pi s\Leftrightarrow s^2+(3-\pi)s+1>0
\end{align*}
for all $s>0$ implies \eqref{eq:ln_decreasing} is decreasing. The discriminant of $s\mapsto s^2+(3-\pi)s+1$ factors as $(\pi-5)(\pi-1)<0,$ so we are done.

%\printbibliography

\bibliographystyle{plain}
\bibliography{rvfl}

%\begin{thebibliography}{9}

%\bibitem{mcshane} E.J.\ McShane. \emph{Extension of range of functions}. Bull.\ Am.\ Math.\ Soc., 40.12, 1934.

%\bibitem{wendel} J.G.\ Wendel. \emph{Note on the Gamma Function}. Am.\ Math.\ Mon., 55.9, 1948.

%\bibitem{ehm}  Werner Ehm, Tilmann Gneiting and Donald Richards. \emph{Convolution Roots of Radial Positive Definite Functions with Compact Support}. Trans.\ Am.\ Math.\ Soc., 356.11, 2004.

%\bibitem{dlmfbessel} \url{https://dlmf.nist.gov/10.7#E3}

%\bibitem{qu} C.K.\ Qu and R.\ Wong. \emph{``Best Possible'' Upper and Lower Bounds for the Zeros of \\ the Bessel Function $J_\nu(x)$}. Trans.\ Am.\ Math.\ Soc., 351.7, 1999. 

%\bibitem{dlmfincompletegamma} \url{https://dlmf.nist.gov/8}

%\bibitem{dlmfgamma} \url{https://dlmf.nist.gov/5.6#E1}

%\bibitem{dlmf} \url{https://dlmf.nist.gov}

%\bibitem{border} Charalambos D.\ Aliprantis, Kim C.\ Border. \emph{Infinite Dimensional Analysis: A} \emph{Hitchhiker's Guide}. 3rd ed., Springer (2006).

%\end{thebibliography}

\end{document}